\RequirePackage{etoolbox}

\newcommand{\type}{report}

\newcommand{\verbosity}{short}

\newcommand{\mode}{final}

\newcommand{\apx}{proofs}

\newcommand{\numberlevel}{section}

\newcommand{\counterlevel}{same}

\newcommand{\colorthm}{false}

\newif\ifsamecounter\ifdefstring{\counterlevel}{same}{\samecountertrue}{\samecounterfalse}

\newif\iflong\ifdefstring{\verbosity}{long}{\longtrue}{\longfalse}

\newif\ifshort\ifdefstring{\verbosity}{long}{\shortfalse}{\shorttrue}

\newif\iffinal\ifdefstring{\mode}{final}{\finaltrue}{\finalfalse}

\newif\ifdraft\ifdefstring{\mode}{draft}{\drafttrue}{\draftfalse}

\newif\ifreview\ifdefstring{\mode}{review}{\reviewtrue}{\reviewfalse}

\newif\ifsubmission\ifdefstring{\mode}{submission}{\submissiontrue}{\submissionfalse}

\newcommand{\bibstylename}{unsrtnat}

\newcommand{\contact}[1]{}
\newcommand{\affiliation}[1]{\\#1}

\ifdefstring{\type}{conference}{%

    \documentclass{article} 

    \usepackage{aaai21}  
    \usepackage{times}  
    \usepackage{helvet} 
    \usepackage{courier}  
    \usepackage[hyphens]{url}  
    \usepackage{graphicx} 
    \urlstyle{rm} 
    \usepackage{natbib}  
    \usepackage{caption} 
    \frenchspacing  
    \setlength{\pdfpagewidth}{8.5in}  
    \setlength{\pdfpageheight}{11in}  

    \setcounter{secnumdepth}{2} 

    \renewcommand{\bibstylename}{aaai21} 

    \pdfinfo{
        /Title (Exact Reduction of Huge Action Spaces in General Reinforcement Learning)
        /Author (Sutan J. Majeed, Marcus Hutter)
        /TemplateVersion (2021.2)
    }

    \author{
        Sultan J. Majeed\textsuperscript{\rm 1}\thanks{The contact author.}\\
        Marcus Hutter\textsuperscript{\rm 2}\\
    }
    \affiliations{
        \textsuperscript{\rm 1,2}Research School of Computer Science, ANU, Canberra\\
        \textsuperscript{\rm 2}Google DeepMind, London\\
        \textsuperscript{\rm 1}www.sultan.pk, \textsuperscript{\rm 2}www.hutter1.net

    }
}{\ifdefstring{\type}{journal}{%

        \documentclass{report} 

        \renewcommand{\bibstylename}{apalike}

    }{

        \documentclass[12pt]{xarticle}

        \usepackage[a4paper, total={8.5in, 11in},margin=1in]{geometry}

        \usepackage{libertine}

        \renewcommand{\bibstylename}{unsrtnat}

        \emergencystretch 3em

        \author{
            Sultan J. Majeed\thanks{The contact author.}
            \contact{sultan.pk} \affiliation{Research School of Computer Science, ANU}
            \and
            Marcus Hutter
            \contact{hutter1.net} \affiliation{Google DeepMind \& Research School of Computer Science, ANU}
        }

}}

\usepackage{amsmath, amsthm, amssymb, thmtools}


\usepackage[
nottoc
]{tocbibind}

\usepackage{tocloft}
\usepackage{xcolor}

\usepackage{varioref}

\usepackage[
colorlinks = true,
linkcolor = blue!80!black,
urlcolor  = green!50!black,
citecolor = blue!80!black,
anchorcolor = blue!80!black,
]{hyperref}
\usepackage{xstring}

\usepackage{subcaption}

\usepackage{blindtext}

\usepackage{enumitem}

\usepackage[
\ifdraft\else
disable,
\fi
colorinlistoftodos,
prependcaption,
textsize=tiny,
color=red!20,
linecolor=red!20,
shadow,
backgroundcolor=red!25,
bordercolor=red
]{todonotes}
\presetkeys{todonotes}{fancyline}{}

\ifreview
\usepackage{lineno}\linenumbers
\BeforeBeginEnvironment{dmath}{\begin{nolinenumbers}}%
    \AfterEndEnvironment{dmath}{\end{nolinenumbers}}
\fi

\iffinal\else\ifsubmission\else
\usepackage{draftwatermark}
\ifdraft
\SetWatermarkText{Draft Copy - Do Not Distribute}
\SetWatermarkColor[rgb]{0.8,0,0}
\else\ifreview
\SetWatermarkText{Review Copy - Do Not Distribute}
\SetWatermarkColor[rgb]{0,0,0.8}
\fi\fi
\SetWatermarkFontSize{2em}
\SetWatermarkAngle{0.0}
\SetWatermarkHorCenter{\paperwidth/2}
\SetWatermarkVerCenter{2em}
\SetWatermarkScale{1.0}
\fi\fi

\usepackage{tikz}
\usetikzlibrary{automata,positioning,shapes,shapes.geometric,arrows,fit,calc}
\usetikzlibrary{decorations.pathmorphing}

\usepackage{pgfplots}

\usepackage{nameref}

\usepackage{algorithm}
\usepackage[noend]{algpseudocode}
\algnewcommand\Input{\item[{\textbf{Input:}}]}
\algnewcommand\Output{\item[{\textbf{Output:}}]}

\ifdefempty{\apx}{}{
}

\usepackage{flexisym}
\usepackage{breqn}

\usepackage[square,sort,comma,numbers]{natbib}
\usepackage{usebib}
\bibinput{references} 

\usepackage[nameinlink]{cleveref}

\crefname{subequation}{case}{cases}
\Crefname{subequation}{Case}{Cases}

\ifdraft\usepackage{showkeys}\fi%

\usepackage[skins,breakable]{tcolorbox}
\tcbset{
    breakable,
    enhanced,
    colback=black!5!white,
    colframe=black!75!white,
    boxrule=1pt,
    sharp corners=all,
    before skip=\topsep,
    after skip=\topsep,
}

\makeatletter

\ifsamecounter

\newtheorem{thm}{Should Not Be Visible}[\numberlevel]
\ifx\c@theorem\undefined \newtheorem{theorem}[thm]{Theorem}\fi
\ifx\c@lemma\undefined \newtheorem{lemma}[thm]{Lemma} \fi
\ifx\c@proposition\undefined \newtheorem{proposition}[thm]{Proposition} \fi
\ifx\c@definition\undefined \newtheorem{definition}[thm]{Definition} \fi
\ifx\c@remark\undefined \newtheorem{remark}[thm]{Remark} \fi
\ifx\c@corollary\undefined \newtheorem{corollary}[thm]{Corollary} \fi
\ifx\c@example\undefined \newtheorem{example}[thm]{Example}\fi
\ifx\c@assumption\undefined \newtheorem{assumption}[thm]{Assumption}\fi
\ifx\c@conjecture\undefined \newtheorem{conjecture}[thm]{Conjecture}\fi

\else

\ifx\c@theorem\undefined \newtheorem{theorem}{Theorem}[\numberlevel] \fi
\ifx\c@lemma\undefined \newtheorem{lemma}{Lemma}[\numberlevel] \fi
\ifx\c@proposition\undefined \newtheorem{proposition}{Proposition}[\numberlevel] \fi
\ifx\c@definition\undefined \newtheorem{definition}{Definition}[\numberlevel] \fi
\ifx\c@remark\undefined \newtheorem{remark}{Remark}[\numberlevel] \fi
\ifx\c@corollary\undefined \newtheorem{corollary}{Corollary}[\numberlevel] \fi
\ifx\c@example\undefined  \fi
\ifx\c@assumption\undefined \newtheorem{assumption}{Assumption}[\numberlevel] \fi
\ifx\c@conjecture\undefined \newtheorem{conjecture}{Conjecture}[\numberlevel] \fi

\fi

\makeatother

\ifdefstring{\colorthm}{true}{

    \let\thmOrg\theorem
    \let\endthmOrg\endtheorem
    \renewenvironment{theorem}{\begin{tcolorbox}\thmOrg}{\endthmOrg\end{tcolorbox}}

    \let\lemOrg\lemma
    \let\endlemOrg\endlemma
    \renewenvironment{lemma}{\begin{tcolorbox}\lemOrg}{\endlemOrg\end{tcolorbox}}

    \let\proOrg\proposition
    \let\endproOrg\endproposition
    \renewenvironment{proposition}{\begin{tcolorbox}\proOrg}{\endproOrg\end{tcolorbox}}

    \let\defOrg\definition
    \let\enddefOrg\enddefinition
    \renewenvironment{definition}{\begin{tcolorbox}\defOrg}{\enddefOrg\end{tcolorbox}}

    \let\corOrg\corollary
    \let\endcorOrg\endcorollary
    \renewenvironment{corollary}{\begin{tcolorbox}\corOrg}{\endcorOrg\end{tcolorbox}}

    \let\asmOrg\assumption
    \let\endasmOrg\endassumption
    \renewenvironment{assumption}{\begin{tcolorbox}\asmOrg}{\endasmOrg\end{tcolorbox}}

    \let\conOrg\conjecture
    \let\endconOrg\endconjecture
    \renewenvironment{conjecture}{\begin{tcolorbox}\conOrg}{\endconOrg\end{tcolorbox}}

    \let\rmkOrg\remark
    \let\endrmkOrg\endremark
    \renewenvironment{remark}{\begin{tcolorbox}\rmkOrg}{\endrmkOrg\end{tcolorbox}}

}{}

\usepackage{agi-research}

\usepackage{mathtools}
\usepackage{aligned-overset}

\newcommand\restrict[2]{{
        \left.\kern-\nulldelimiterspace 
        #1 
        \vphantom{\big|} 
        \right|_{#2} 
}}

\usetikzlibrary{arrows,positioning}
\tikzset{
    >=stealth',
    roundbox/.style={
        rectangle,
        rounded corners,
        draw=black, very thick,
        text width=6.5em,
        minimum height=2em,
        text centered},
    thickarrow/.style={
        ->,
        thick,
        shorten <=2pt,
        shorten >=2pt,}
}

\usetikzlibrary{fit,positioning}

\title{
    Exact Reduction of Huge Action Spaces in General Reinforcement Learning%
}

\begin{document}


\ifdefempty{\mode}{
    \pagestyle{empty}
    \newpage
    \listoftodos[Todo \& Notes]
    \newpage
    \pagestyle{plain}
    \setcounter{page}{1}
}{}


\maketitle


\begin{abstract}
    The reinforcement learning (RL) framework formalizes the notion of learning with interactions.
    Many real-world problems have large state-spaces and/or action-spaces such as in Go, StarCraft, protein folding, and robotics or are non-Markovian, which cause significant challenges to RL algorithms.
    In this work we address the large action-space problem by sequentializing actions, which can reduce the action-space size significantly, even down to two actions at the expense of an increased planning horizon. We provide explicit and exact constructions and equivalence proofs for all quantities of interest for arbitrary history-based processes. In the case of MDPs, this could help RL algorithms that bootstrap.
    In this work we show how action-binarization in the non-MDP case can significantly improve Extreme State Aggregation (ESA) bounds. ESA allows casting any (non-MDP, non-ergodic, history-based) RL problem into a fixed-sized non-Markovian state-space with the help of a surrogate Markovian process. On the upside, ESA enjoys similar optimality guarantees as Markovian models do. But a downside is that the size of the aggregated state-space becomes exponential in the size of the action-space. In this work, we patch this issue by binarizing the action-space. We provide an upper bound on the number of states of this binarized ESA that is logarithmic in the original action-space size, a double-exponential improvement.
\end{abstract}


\section{Introduction}

The reinforcement learning (RL) setting can be described by an agent-environment interaction \cite{Sutton2018}. The agent $\Pi$ has an action-space $\A$ to choose its actions from while the environment $P$ reacts to the action by dispensing an observation and a reward from the sets $\O$ and $\R \subseteq \SetR$, respectively, see \Cref{fig:interaction}. For simplicity, we assume that these sets are finite and hence the rewards are bounded. Even with these restrictions, the problem of RL does not trivialize, i.e.\ the agent can not learn the optimal behavior without further structure. Under a suitable definition of the ``state'' of environment, the resultant set of states might be huge or even infinite \cite{Powell2011}.

\begin{figure}[!ht]
    \centering
    \begin{tikzpicture}[
    node distance = 7mm and -3mm,
    innernode/.style = {draw=black, thick, fill=gray!30,
        minimum width=2cm, minimum height=0.5cm,
        align=center},
    outernode/.style = {draw=black, thick, rounded corners, fill=none,
        minimum width=1cm, minimum height=0.5cm,
        align=center},
    endpoint/.style={draw,circle,
        fill=gray, inner sep=0pt, minimum width=4pt},
    arrow/.style={->,thick,rounded corners},
    point/.style={circle,inner sep=0pt,minimum size=2pt,fill=black},
    skip loop/.style={to path={-- ++(#1,0) |- (\tikztotarget)}},
    every path/.style = {draw, -latex}
    ]
    \node (start) {Start};
    \node (h) [innernode]{History};
    \node (phi) [innernode, below=of h]{Abstraction $(\phi)$};
    \node (pi) [innernode, below=of phi]      {Policy $(\Pi)$};
    \node [outernode, align=left, inner sep=0.5cm, fill=none, fit=(h) (phi) (pi)] (agent) {};
    \node[below right, inner sep=3pt, fill=none] at (agent.north west) {Agent};
    \node[outernode, left=120pt of agent, fit=(agent.north)(agent.south), inner sep=0pt] (env) {};
    \node[below right, inner sep=0pt, fill=none, rotate=90, anchor=center] at (env) {Environment $(P)$};
    \node[endpoint, above= -2pt of env] (or_env) {};
    \node[endpoint, below= -2pt of env] (a_env) {};
    \node[endpoint, below= -2pt of agent] (a_agent) {};
    \node[endpoint, above= -2pt of agent] (or_agent) {};

    \path (a_agent) edge[arrow,bend left] node[below]{$a_t$} (a_env);
    \path (or_env) edge[arrow, bend left] node[above]{$o_{t+1}r_{t+1}$} (or_agent);
    \path (or_agent) edge[arrow] node[right]{$o_{t+1}r_{t+1}$} (h);
    \path (h) edge[arrow] node[above=0.5pt,midway,name=h_phi,point]{} node[right]{$h_t$} (phi);
    \path (phi) edge[arrow] node[left]{$\phi(h_t)$} (pi);
    \path (pi) edge[arrow] node[above=0.5pt,midway,name=pi_a,point]{} node[left]{$a_t$} (a_agent);
    \path (pi_a) edge[arrow, skip loop=1.5cm] (h.east);
    \path (h_phi) edge[->, skip loop=-1.5cm, thick, rounded corners] (h.west);
    \end{tikzpicture}
    \caption{The agent-environment interaction.}
    \label{fig:interaction}
\end{figure}
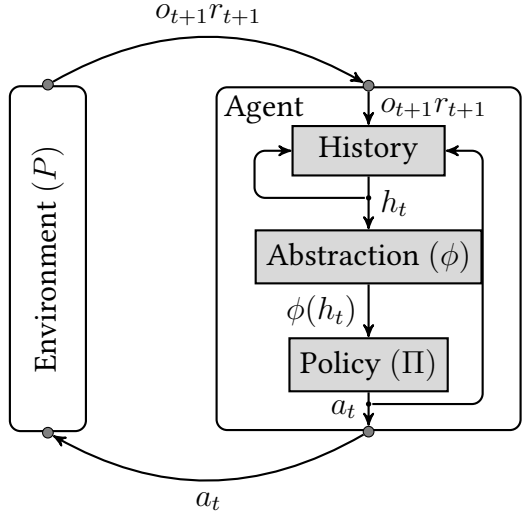

The problem of RL is plagued with the curse of dimensionality. The sizes of an appropriately defined set of states\footnote{We prefer not to call the set of observations $\O$ the set of states. This set becomes a set of states under strong assumptions, see \Cref{sec:esa} for more details on this distinction.} and action-space play an important role in the choice of algorithms, architectures and solution techniques used to solve the task \cite{Sutton2018}.

It is usually required to produce a relatively small set of states to make the problem tractable \cite{Lattimore2014}. There are many ways to achieve such approximations/abstractions, e.g.\ state aggregation \cite{Li2006,Abel2016,Hutter2016}, homomorphism \cite{Majeed2019}, linear function approximation \cite{Bertsekas1996}, or neural networks \cite{Mnih2015} just to name a few. The usual assumption for such abstractions is that they try to produce a Markovian representation of the environment, which is known as a Markov Decision Process (MDP) \cite{Hutter2009}. In an MDP  the most recent (abstract) observation is sufficient to predict \emph{any} future event\footnote{An event is any set of action-observation-reward sequences.}, but not all events are equally valuable. Some events might not lead to a high rewarding state and/or some distinctions are not really necessary to perform well \cite{McCallum1995}. For example, the agent might end up experiencing two completely different streams of observations with the same reward structure. An algorithm which tries to produce a Markovian representation would try to make this ``unnecessary'' distinction.

A lot can be achieved in terms of the representation power if an algorithm only makes ``useful'' distinctions, i.e.\ the distinctions (or states) which respect the reward structure \cite{McCallum1995,Hutter2016,Majeed2019}.
In some cases, such ``useful'' but non-Markovian abstractions reduce the effective state-space dramatically. Usually, a smaller state-space facilitates faster learning \cite{Strehl2009,Lattimore2014}.

The usual methods of state or action-space reductions either (1) reduce the problem to a fixed size where the quality of reduction deteriorates as the original problem becomes more complicated, or (2) provide a problem-specific reduction which usually grows, albeit much slower, as the original problem grows \cite{Powell2011}. In Markovian abstractions the size of the state-space grows with the size of the observation and reward spaces. For example, if an MDP abstraction produces states from some low-resolution images then we need more states to handle high-resolution versions of the input images because the high-resolution images need a bigger transition matrix to predict the next image. However, it is perfectly plausible that the increased resolution might not be ``useful'' to achieve better rewards.

To the best of our knowledge, extreme state aggregation (ESA), a non-MDP abstraction framework, is the only method which provides a provable upper bound on the size of required state-space uniformly\footnote{Which depends only on the size of the action-space, discount factor and the optimality gap.} for all problems \cite{Hutter2016}. However, a \emph{downside} of ESA is that the size of the aggregated state-space is \emph{exponential} in the size of the action-space, see \Cref{thm:esa}.
In this paper, we move the research further in this direction. We provide a variant of ESA that can help provide much more compact representations as compared to MDP abstractions. Our approach improves the key upper bound on the size of the state-space in the original ESA framework.

The key trick to achieve this improvement is to sequentialize the actions. Often $\A$ already has a natural vector structure $\B^d$, e.g.\ real valued activators in robotics ($\B = \SetR$) or (padded) words ($\B = \{a, \dots, z, \textvisiblespace\}$), or more generally $\B_1 \times \dots \times \B_d$, where $\B$ denotes a finite set of \emph{decision symbols}.
In this case, sequentialization is natural, but one may further want to binarize $\B$ to $\SetB^{d'}$ esp. for ESA (\Cref{thm:bin-esa}).
If actions are (converted to) $\B$-ary strings,
the RL agent could execute the action ``bits'' sequentially with fictitious dummy observations in-between.

The example in \Cref{fig:example} provides a naive way of \emph{sequentializing} the actions in an MDP. Apparently, it might seem that such sequentialization of the action-space would be of no help, as the state-space would blow up, and it is simply shifting the problem from the actions to the states. However, we prove that this can be avoided. Most importantly, the universal upper bound on the effective state-space of ESA remains valid. Our scheme of sequentializing the actions achieves a \emph{double exponentially} improved bound; compare \Cref{thm:bin-esa} with \Cref{thm:esa}.

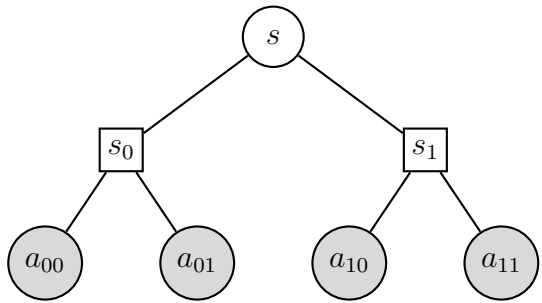
\begin{figure}[h]
   \centering
   \begin{tikzpicture}[level/.style={sibling distance=40mm/#1,thick},square/.style={regular polygon,regular polygon sides=4,minimum size=1mm, inner sep=0,thick}]
   \node [circle,draw,inner sep=0, minimum size=8mm,thick] (z){$s$}
   child {node [square,draw,inner sep=0, minimum size=8mm,thick] (a) {$s_0$}
       child {node [circle,draw,fill=gray!30,thick] (b) {$a_{00}$}}
       child {node [circle,draw,fill=gray!30,thick] (g) {$a_{01}$}}
   }
   child {node [square,draw,inner sep=0, minimum size=8mm,thick] (j) {$s_1$}
       child {node [circle,draw,fill=gray!30,thick] (k) {$a_{10}$}}
       child {node [circle,draw,fill=gray!30,thick] (l) {$a_{11}$}}
   };
   \end{tikzpicture}
   \caption{A simple sequentialization example in an MDP. To see how the actions sequentialized, consider an agent which has to choose among four alternatives, e.g.\ $\A = \{a_{00}, a_{01}, a_{10}, a_{11}\}$. Let the agent receive a state signal $s$ from the environment. It first decides between a partition of actions, say two actions each, $\{a_{00}, a_{01}\}$ and $\{a_{10}, a_{11}\}$. \emph{After} it has decided on the bifurcation, the \emph{extended} state becomes $s_x$, where $x$ is the decision of the first stage. Now the agent \emph{on} this \emph{extended} state $s_x$ makes its second decision to choose from the \emph{short-listed} set of actions. This way, the agent only selects among \emph{two} alternatives at each stage by \emph{tripling} the effective state-space.}
   \label{fig:example}
\end{figure}

Along the way, we also establish some other key results, which are interesting and useful on their own. We provide explicit and exact constructions and equivalence proofs for all quantities of interest (\Cref{sec:bin-esa}) for arbitrary history-based processes, which are then used to double-exponentially improve the previous ESA bound (\Cref{thm:bin-esa}).
In the special case of MDPs, we show that through a sequentialized scheme (of augmenting observations with partial decision vectors) the resultant ``sequentialized process'' preserves the Markov property (\Cref{pro:mdpismdp}), which should help RL algorithms that bootstrap, though demonstrating or proving this is left for future work. Moreover in \Cref{lem:uplift}, we prove that the stipulated sequentialization scheme preserves near-optimality, i.e.\ a near-optimal policy of the sequentialized process is also near-optimal in the original process.

The rest of the paper is organized as follows. \Cref{sec:notation} puts down the necessary notation. In \Cref{sec:setup}, we formally set up the problem. A framework to sequentialize actions is provided in \Cref{sec:bin-esa}, along with other key results (\Cref{pro:mdpismdp,lem:uplift}). In \Cref{sec:esa} we combine our sequentialization framework with ESA to improve the upper bound on the size of the state-space (\Cref{thm:bin-esa}). We conclude the paper in \Cref{sec:conclusion} with an outlook.

\section{Notation}\label{sec:notation}

This paper is notation heavy, but we use a consistent notation through out.
The set of natural numbers is $\SetN := \{1, 2, \dots\}$, $\SetB := \{0,1\}$ is a set of binary symbols, and $\SetR$ is the set of reals. We denote by $\Dist(X)$ the set of probability distributions over any set $X$. The concatenation of two objects (or strings) is expressed through juxtaposition, e.g.\ $xy$ is a concatenation of $x$ and $y$. We express a finite string with boldface, e.g.\ $\v x = x_1 x_2 \dots x_{\abs{\v x}}$ where $\abs{{}\cdot{}}$ is used to denote the length or cardinality of the object. The individual members of a string or a vector may be accessed as $\v x_i = x_i$ for any $i \leq \abs{\v x}$. A substring of length $i \leq \abs{\v x}$ is denoted as $\v x_{\leq i} = x_1x_2\dots x_{i}$ and $\v x_{<i} = x_1x_2\dots x_{i-1}$. We interchangeably use the same notation for vectors and strings, e.g.\ $\v x \in \B^d$ is a $d$-dimensional $\B$-ary decision vector which may also be expressed as a string. This choice simplifies the notation and saves redundant variables. If a variable is time-indexed, we express the continuation of the variable with a prime on it, e.g.\ if $x := x_{t}$ then $x' := x_{t+1}$ where $:=$ denotes equality by definition. A small scaler value (usually the error tolerance) is denoted by $\eps > 0$.
A different member of the same set is expressed with a dot on it, e.g.\ $x, \d x \in \B$. We express the fact of $\v x$ being a prefix of $\v y$ by $\v x \sqsubseteq \v y$ or $\v y \sqsupseteq \v x$. Moreover, $\vo{xy}$ represents a vector that is point-wise joined, i.e.\ $\vo{xy}_i := \v x_i \v y_i$.

\section{Problem Setup}\label{sec:setup}

This section provides the formal foundations for us to build up to the main result of this work. We formulate the problem of RL from the ground up without starting from the usual Markovian assumption \cite{Sutton2018}. We formalize a history-based RL setup. After formalizing the (general) RL problem, we set up the scheme of sequentializing the decision-making process to reduce the effective action-space for the agent. Especially for our main result about ESA, we sequentialize the action-space to \emph{binary} decisions.

Although this work assumes very little about the RL problem, we assume that the size of the action-space is finite and $\abs{\A} = \abs{\B}^d$ for some $d \in \SetN$. The latter assumption is not restrictive, as we can extend the set of actions by repeating some of the actions. It is important to note that these repeated actions should be labeled distinctly. This way we can have a bijection between the original (extended) action-space and the sequentialized one. For example, let an action set be $\{a_1,a_2,a_3,a_4,a_5\}$. One possible extended set, with repetition, for $\abs{\B} = 2$ and $d = 3$ is $\A := \{a_1,a_2,a_3,a_4,a_5,a_{5_1},a_{5_2},a_{5_3}\}$. Where, the actions $a_{5_i}$ for $i \leq 3$ are functionally the same as $a_5$, i.e.\ taking $a_5$ or any $a_{5_i}$ action has the same effect, but they are labeled distinctly.

Note that continuous action-spaces could be approximately sequentialized/binarized by using the binary expansion of reals to some desired precision, say $\delta$. Our main bound will only depend logarithmically on $\delta$.

\subsection{General Reinforcement Learning}

We consider a general reinforcement learning (GRL) setup where the agent keeps the complete history of interaction \cite{Hutter2009}. The (infinite\footnote{For simplicity, we assume the interaction never stops. We do not consider the case where the agent or the environment can stop responding. It complicates the modeling beyond the scope of this work.}) interaction produces an infinite history. Recall that $\O, \R$ and $\A$ represent some \emph{finite} sets of observations, rewards and actions, respectively. This also implies that the rewards are bounded. The set of all finite histories\footnote{Note that this set (of underlying ``state'' space) is (countably) infinite.} is denoted by
\beq\label{eq:history}
\H := \bigcup_{t=1}^\infty \underbrace{\O \times \R \times \A \times \dots \times \O \times \R \times \A}_{(t-1)-\text{step interactions}}\times\O\times\R
\eeq
which is used to express most of the quantities in our setup, e.g.\ the environment, agent, and value functions. Note that the history set $\H$ does not contain the empty history. This is a design choice we make to be consistent with the standard RL setup \cite{Sutton2018} where the initial state (in our case the initial observation and reward) is chosen by some ``initial'' distribution.

Formally, the environment $P$ is a (conditional) probability function such that $P : \H \times \A \to \Dist(\O \times \R)$. Similarly, the agent is also expressible as a (conditional) distribution on the action-space, i.e.\ $\Pi: \H \to \Dist(\A)$. For a fixed policy $\Pi$ the expected discounted future sum of rewards is the value of the policy. At any history $h \in \H$ and action $a \in \A$ the action-value function (or Q-function) is expressed as
\beq\label{eq:bellman}
Q^\Pi(h,a) := \sum_{o'r'} P(o'r'|ha) \left(r' + \g V^\Pi(hao'r')\right)
\eeq
where $V^\Pi(h) := \sum_a Q^\Pi(h,a)\Pi(a|h)$ is the value function of $\Pi$ and $0\leq\g<1$ is the discount factor. \Cref{eq:bellman} is known as the Bellman equation (BE). The optimal behavior (or policy) is the one which achieves the maximum value for all histories, i.e.\ $\Pi^*(h) :\in \argmax_{\Pi} V^\Pi(h)$. The optimal value (and action-value) functions, $V^* := V^{\Pi^*}$ and $Q^* := Q^{\Pi^*}$, of this optimal policy satisfy the following optimal Bellman equation (OBE) \cite{Hutter2016,Sutton2018}.
\beq
Q^*(h,a) := \sum_{o'r'} P(o'r'|ha) \left(r' + \g V^*(hao'r')\right)\label{eq:obe}
\eeq
where $V^*(h) := \max_a Q^*(h,a)$. The agent defined in this sub-section works with the original action-space $\A$ and keeps the histories from $\H$. In the next sub-section, we formulate another agent which only works in the ``sequentialized'' action-space, i.e.\ it takes decisions in a sequence of $\B$-ary choices, and responds \emph{only} to the histories generated by this $\B$-ary interaction, see \Cref{fig:bianary-interaction}. In the extreme case, this agent may only take binary decisions by sequentializing the action-space to binary sequences, i.e.\ $\B = \SetB$.

\subsection{Sequential Decisions}

We want to transform the action-space into a sequence of $\B$-ary decision code words, which are decided sequentially. To map the actions between the original action-space and the sequentialized decision-space, we define a pair of encoder and decoder functions. Let $C$ be any encoding function that maps the actions to a $\B$-ary decision code of length $d$, i.e.\ $C : \A \to \B^d$. A decoder function $D : \B^d \to \A$ sends the $\B$-ary decision sequences generated by $C$ back to the actions in the (original) action-space.
In this work, the choices of $C$ and $D$ do not matter\footnote{The choice could matter in practical implementation of such agents. For example, a clever choice of such functions might produce sparse $\B$-ary decision sequences for the optimal actions, hence it may facilitate in learning such optimal $\B$-ary decision sequences.} as long as they are bijections such that $D(C(\A)) = \A$.

This sequentialization of the action-space changes the interaction history. The generated histories are no longer members of $\H$. The goal of this paper is to argue that an agent can still work with the sequentialized histories only. The agent can plan, learn and interact with the environment using $\B$-ary actions and keeping sequentialized histories. Hence, the agent can be agnostic to the original action-space and with the state provided through an appropriate abstraction it can only take $\B$-ary decisions at every time step, see \Cref{fig:bianary-interaction}.

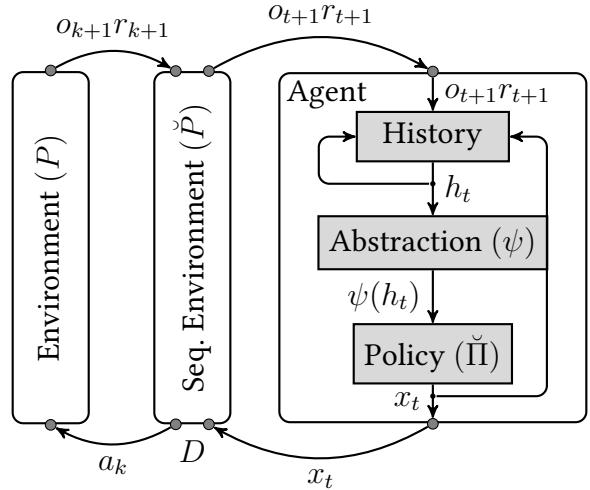
\begin{figure}[!ht]
    \centering
    \begin{tikzpicture}[
    node distance = 7mm and -3mm,
    innernode/.style = {draw=black, thick, fill=gray!30,
        minimum width=2cm, minimum height=0.5cm,
        align=center},
    outernode/.style = {draw=black, thick, rounded corners, fill=none,
        minimum width=1cm, minimum height=0.5cm,
        align=center},
    endpoint/.style={draw,circle,
        fill=gray, inner sep=0pt, minimum width=4pt},
    arrow/.style={->,thick,rounded corners},
    point/.style={circle,inner sep=0pt,minimum size=2pt,fill=black},
    skip loop/.style={to path={-- ++(#1,0) |- (\tikztotarget)}},
    every path/.style = {draw, -latex}
    ]
    \node (start) {Start};
    \node (h) [innernode]{History};
    \node (phi) [innernode, below=of h]{Abstraction $(\psi)$};
    \node (pi) [innernode, below=of phi]      {Policy $(\u \Pi)$};
    \node [outernode, align=left, inner sep=0.5cm, fill=none, fit=(h) (phi) (pi)] (agent) {};
    \node[below right, inner sep=3pt, fill=none] at (agent.north west) {Agent};
    \node[outernode, left=90pt of agent, fit=(agent.north)(agent.south), inner sep=0pt] (bin_env) {};
    \node[outernode, left=143pt of bin_env, fit=(agent.north)(agent.south), inner sep=0pt] (env) {};
    \node[below right, inner sep=0pt, fill=none, rotate=90, anchor=center] at (env) {Environment $(P)$};
    \node[below right, inner sep=0pt, fill=none, rotate=90, anchor=center] at (bin_env) {Seq. Environment $(\u P)$};
    \node[below= 2pt of bin_env] (D) {$D$};
    \node[endpoint, above= -2pt of env] (or_env) {};
    \node[endpoint, below= -2pt of env] (a_env) {};
    \node[endpoint, left=4pt of bin_env.north] (or_bin_env_in) {};
    \node[endpoint, right=4pt of bin_env.north] (or_bin_env_out) {};
    \node[endpoint, left=4pt of bin_env.south] (a_bin_env_out) {};
    \node[endpoint, right=4pt of bin_env.south] (a_bin_env_in) {};
    \node[endpoint, below= -2pt of agent] (a_agent) {};
    \node[endpoint, above= -2pt of agent] (or_agent) {};

    \path (a_agent) edge[arrow,bend left] node[below]{$x_t$} (a_bin_env_in);
    \path (a_bin_env_out) edge[arrow,bend left] node[below]{$a_k$} (a_env);
    \path (or_bin_env_out) edge[arrow, bend left] node[above]{$o_{t+1}r_{t+1}$} (or_agent);
    \path (or_env) edge[arrow, bend left] node[above]{$o_{k+1}r_{k+1}$} (or_bin_env_in);
    \path (or_agent) edge[arrow] node[right]{$o_{t+1}r_{t+1}$} (h);
    \path (h) edge[arrow] node[above=0.5pt,midway,name=h_phi,point]{} node[right]{$h_t$} (phi);
    \path (phi) edge[arrow] node[left]{$\psi(h_t)$} (pi);
    \path (pi) edge[arrow] node[above=0.5pt,midway,name=pi_a,point]{} node[left]{$x_t$} (a_agent);
    \path (pi_a) edge[arrow, skip loop=1.5cm] (h.east);
    \path (h_phi) edge[->, skip loop=-1.5cm, thick, rounded corners] (h.west);
    \end{tikzpicture}
    \caption{The agent-environment interaction through the sequentialization scheme. Note that the sequentialized-environment block (or a $\B$-ary ``mock'') manages two different time-scales $t$ and $k$. It is simply a buffer block which knows (de)coders $C$ and $D$ (see text for details). It buffers the input $\B$-ary actions and dispatches the buffered observation and reward. Once a complete $\B$-ary decision sequence is produced by the agent the $\B$-ary mock decodes the encoded actions to the original environment to continue the interaction loop. We can consider this sequentialized environment as a ``middle layer'' between the agent and the original environment.}
    \label{fig:bianary-interaction}
\end{figure}

We construct a history transformation function which maps the original histories from $\H$ to some sequentialized histories in $\u \H$, where
\beq
\u \H := \bigcup_{t=1}^\infty \underbrace{\O \times \R \times \B \times \dots \times \O \times \R \times \B}_{(t-1)-\text{step interactions}}\times\O\times\R
\eeq

It is worth noting that $\u \H$ does not (directly) contain any information about $\A$, cf.\ \Cref{eq:history}. The agent experiencing histories from this set would not be aware of $\A$.

\begin{definition}[History transformation function]
    A history transformation function is expressed with $g : \H \to \u \H$. The map is recursively defined for any history $h$, action $a$, next observation $o'$ and next reward $r'$ as
    \beq
    g(hao'r') := g(h)\v x_1or_\bot\v x_2or_\bot \dots \v x_do'r' \ \text{and} \ g(e) := e
    \eeq
    where $\v x := C(a)$, $o$ is the last observation of the history $h$, $e$ denotes the ``initial'' history\footnote{The initial history $e \in \O \times \R$ is similar to the initial state in standard RL. It is dispatched by environment without any input at the start.}, and $r_\bot$ is any fixed real-value. In this work, we assume\footnote{This assumption is not much of a restriction, if $r_\bot \notin \R$ then we can extend the reward space by $\R \cup \{r_\bot\}$.} that $r_\bot \in \R$ and $r_\bot = 0$.
\end{definition}

In the above construction, we chose to repeat the last observation $o$ in between the real interactions with the environment. This is not the only possible choice, we can choose a \emph{dummy} observation $o_\bot \in \O$ instead without affecting the claims. For brevity, we define $\v o$ and $\v r_\bot$ as $d$-dimensional constant vectors of $o$ and $r_\bot$, respectively. These vectors are then ``welded'' with $\v x$ to succinctly replace $x_1or_\bot\dots x_ior_\bot$ with $\vo{xor_\bot}_{\leq i}$. Note that we do not sequentialize the observations. It can be done, but we believe it is not useful in any way.

However, if the original process $P$ is an MDP, i.e.\ the most recent observation is the state of $P$, then there is another interesting option possible for $o_\bot$: extend the observation space $\O$ with $\O \times \cup_{i=0}^{d-1} \B^{i} =: \t \O$, and let the $\B$-ary mock dispatch an appropriate observation at every partial $\B$-ary decision vector $\v x_{<i}$ as:
\beq\label{eq:tobs}
\t o_\bot \coloneqq (o, \v x_1, \v x_2, \dots, \v x_{i-1}) \in \t \O
\eeq

It is not hard to show that with this sequentialization scheme the resultant sequentialized decision process is also an MDP over $\t \O$, see \Cref{pro:mdpismdp}.
By doing so, we end up with a state-space of size $|\t \O| = \abs{\O}(\abs{\A}-1) \leq \abs{\O \times \A}$. It is clear that this recasting of the original problem might not be very helpful for some Monte-Carlo like tree search methods, however, it might significantly improve the performance of some temporal-difference like algorithms,, e.g.\ Q-learning \cite{Watkins1992}, when applied to huge action-spaces.

Note that $g$ is injective, but it may not be a bijection. There are many sequentialized histories $\tau \in \u \H$ which are not mapped by $g$, i.e.\ there does not exist any history in $\H$ such that $\tau = g(h)$. For such sequentialized histories we define $g\inv(\tau) := \bot$, which formally allows us to talk about $g\inv$ without worrying about it being undefined on some arguments. The choice of this definition is not important. As a matter of fact, there is no particular significance of the symbol $\bot$. What makes this choice insignificant is the fact that the environment does not react until the agent has taken $d$ $\B$-ary actions. Some histories not covered by $g$ are such ``partial'' sequentialized histories where the actual environment does not react. Note that the rewards of the sequentialized setup are zero ($r_\bot := 0$) unless the sequentialized history length is a multiple of $d$, i.e.\ a ``complete'' sequentialized history. See \Cref{fig:example-2} for an example sequentialized/binarized setup for $\B = \SetB$ and $d=2$.

\begin{figure}[h]
    \centering
    \begin{tikzpicture}[level/.style={sibling distance=40mm/#1,thick, inner sep=2pt},square/.style={regular polygon,regular polygon sides=4,minimum size=1mm, inner sep=0,thick}]
        \node [circle,draw,inner sep=0, minimum size=12mm,thick] (z){$\tau$}
        child {node [square,draw, minimum size=12mm,thick] (a) {$\tau 0$}
            child {node [circle,draw,minimum size=12mm,thick] (b) {$\tau 00 o' r'$} edge from parent node[above left] {$x_2 = 0$}}
            child {node [circle,draw,minimum size=12mm,thick] (g) {$\tau 01 o' r'$} edge from parent node[above right] {$x_2 = 1$}}
            edge from parent node[above left] {$x_1 = 0$}
        }
        child {node [square,draw, minimum size=12mm,thick] (j) {$\tau 1$}
            child {node [circle,draw, minimum size=12mm,thick] (k) {$\tau 10 o' r'$} edge from parent node[above left] {$x_2 = 0$}}
            child {node [circle,draw, minimum size=12mm,thick] (l) {$\tau 11 o' r'$} edge from parent node[above right] {$x_2 = 1$}}
            edge from parent node[above right] {$x_1 = 1$}
        };
    \end{tikzpicture}
    \caption{A simple sequentialization/binarization example in a deterministic history-based process. The $\B$-ary/binary decisions are on the edges. For brevity, we do not represent $o_\bot$ and $r_\bot$ in the figure. For example, it should be apparent that $\tau 1 o_\bot r_\bot \equiv \tau 1$. The circles represent complete histories while the squares indicate partial histories.}
    \label{fig:example-2}
\end{figure}
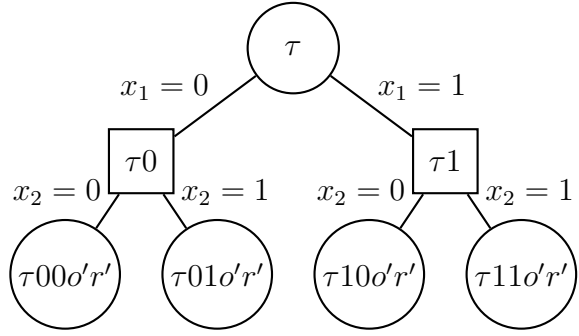

\ifshort\else
TODO: Incorporate MH comment to introduce ``partial'' and ``complete'' histories.
\fi

Any agent which interacts with the environment through this sequentialized scheme would effectively experience the following sequentialized environment.
\begin{definition}[Sequentialized environment]\label{def:uP}
    For any $\B$-ary action $x \in \B$, sequentialized history $\tau \in \u \H$, and any partial extension $\vo{xor_\bot}_{<i}$ for $i \leq d$ the probability of receiving $o'$ and $r'$ as the next observation and reward is as follows:
    \bqan
    \u P(o'r'|\tau\vo{xor_\bot}_{<i} x)
    &:= \begin{cases}
        P(o'r'|ha) &\text{if } \tau\vo{xor_\bot}_{<i} xo'r' = g(hao'r') \\
        1 &\text{if } o'r' = or_\bot \\ &\phantom{}\text{ and } g^{-1}(\tau\vo{x or_\bot}_{<i}  x o' r') = \bot \\
        0 &\text{otherwise}
    \end{cases} \numberthis
    \eqan
    where $P$ is the actual environment.
\end{definition}

As highlighted before, the history $h$ can not be empty, so the above definition is well-defined.

The next step is to define the (action) value functions for this sequentialized agent-environment interaction. Let $\u \Pi$ be a policy such that $\u \Pi : \u \H \to \Dist(\B)$. Then, we define the (action) value functions similar to the original agent-environment interaction case. For any $\tau \in \u \H$ and $x \in \B$, the action-value function is defined as
\beq
\u Q^{\u \Pi}(\tau, x) := \sum_{o'r'} \u P(o'r'|\tau x) \left(r' + \lambda \u V^{\u \Pi}(\tau x o'r')\right)\label{eq:ube}
\eeq
where $\u V^{\u \Pi}(\tau) := \sum_{x \in \B} \u Q^{\u \Pi}(\tau, x)\u \Pi(x|\tau)$ and $\lambda$ is the discount factor of this sequentialized problem. Similar to the original optimal (action) value functions, $\u Q^*$ and $\u V^*$ denote the optimal (action) value functions of the sequentialized problem. The discount factor $\lambda$ plays an important role in trading off the size of the action-space with the planning horizon. Recall that the size of the original action-space is $\abs{\A} = \abs{\B}^d$. Therefore, if the agent has to make $d$ $\B$-ary decisions for each original action the discount factor after $d$ $\B$-ary actions should be $\g$, i.e.\ $\lambda^d = \g$. This implies that $\lambda = \g^{1/d} < 1$ as $\g < 1$ and $d < \infty$.
\ifshort\else See \Cref{fig:relationship-discount-factors} to better understand this relationship of $\lambda$ and $\g$.

\begin{figure}[!ht]
    \centering
    \includegraphics[width=0.4\textwidth]{example-image-c}
    \caption{The relationship between the discount factors $\g$ and $\lambda$.}
    \label{fig:relationship-discount-factors}
\end{figure}
\fi

This completes the problem setup. We have defined an agent $\u \Pi$ which only makes $\B$-ary decisions and reacts to sequentialized histories, see \Cref{fig:bianary-interaction}. As expected, the set of sequentialized histories $\u \H$ is blown out in comparison with $\H$. However, in \Cref{sec:esa}, we show that, under certain non-Markovian abstractions of either $\H$ or $\u \H$, this expansion is not ``harmful''.

\section{Sequentialized Processes and Values}\label{sec:bin-esa}

In this section we formally define the sequentialized process and related value functions. But first we need a couple of important quantities to state our main results. For any $\B$-ary vector $\v x \in \B^i$ where $i \leq d$, we define $\A(\v x) := \{ a \in \A : \v x \sqsubseteq C(a) \}$ a \emph{restricted} set of actions. Moreover, for any history $h$, an action-value function maximizer on this restricted set is defined as $\Pi^*(h,\v x) \in \argmax_{a \in \A(\v x)} Q^*(h, a)$.

We start off the section by noting an important relationship between the sequentialized process and the original process.

\begin{proposition}[Sequentialized Process]\label{pro:uPtoP}
    For any $o' \in \O, r' \in \R, h \in \H$ and $D(\v x) =: a \in \A$, the following relationship holds between $\u P$ and $P$:
    \beq
    \u P(o'r'|g(h)\vo{xor_\bot}_{<d} \v x_d) = P(o'r'|ha)
    \eeq
\end{proposition}
\begin{proof}
    The proof trivially follows from \Cref{def:uP} by evaluating the definition at $i = d$ with $D(\v x_{<d} \v x_d) = a$.
\end{proof}

When the original process is an MDP then there exists a sequentialization scheme such that the sequentialized process is also Markovian.

\begin{theorem}[Sequentialization preserves Markov property]\label{pro:mdpismdp}
    If $P$ is an MDP over $\O$, and the observations from the $\B$-ary mock are $\t \O \coloneqq  \O \times \cup_{i=0}^{d-1} \B^{i}$, then $\u P$ is also an MDP over $\t \O$.
\end{theorem}
\begin{proof}
    In the case of augmenting the observation space, the definition of $\u P$ becomes slightly more verbose than \Cref{def:uP} as $o_\bot$ is different for each partial history as defined in \Cref{eq:tobs}.
    \bqan
    \u P(\t o'r'|\tau\vo{x\t or_\bot}_{<i} x)
    &:= \begin{cases}
        P(o'r'|ha) &\text{if } \tau\vo{x\t or_\bot}_{<i} x\t o'r' = g(hao'r') \\
        1 &\text{if } \t o'r' = o\v x_{<i} xr_\bot \\ &\phantom{}\text{ and } g^{-1}(\tau\vo{x\t or_\bot}_{<i}  x \t o' r') = \bot \\
        0 &\text{otherwise}
    \end{cases}\numberthis
    \eqan
    for any $i \leq d$, $\t o, \t o' \in \t \O$, and $o \in \O$ is the most recent observation in $h$. At any $h$ the sufficient information is $o$, so $P(o'r'|ha) \equiv P(o'r'|oa)$. Therefore, from the above (expanded) definition of $\u P$, it is clear that:
    \beqn
    \u P(\t o'r'| \tau \vo{x\t or_\bot}_{<i}x) \equiv \u P(\t o'r'| o\v x_{<i}x) = \u P(\t o'r' | \t o x)
    \eeqn
    hence proves the proposition.
\end{proof}

The following proposition proves that the action-values of the ``partial'' histories of the sequentialized problem are related. This fact later helps us to show that these action-value functions respect the Q-uniform structure of the original environment.

\begin{proposition}[$\u Q^*$ $\max$-relationship]\label{prep:expandsion}
    For any sequentialized history $\tau \in \u \H$ such that $g\inv(\tau) \in \H$, the following holds
    \beq
    \max_{x \in \B} \u Q^*(\tau, x) = \lambda^{d-1} \max_{\v x \in \B^d} \u Q^*(\tau \vo{xor_\bot}_{<d}, \v x_d)
    \eeq
\end{proposition}
\begin{proof}
    The proof is straight forward. We successively apply the definition of $\u Q^*$.
    \bqan
    \max_{x_1 \in \B} \u Q^*(\tau, x_1)
    &= \max_{x_1 \in \B} \sum_{o'r'} \u P(o'r'|\tau x_1)\\ &\phantom{=} \left(r' + \lambda \max_{x_2 \in \B} \u Q^{*}(\tau x_1 o'r', x_2)\right)\\
    &\overset{(a)}{=} \lambda \max_{x_1 \in \B} \max_{x_2 \in \B} \u Q^*(\tau x_1 or_\bot, x_2)\\
    &\vdotswithin{=} \text{(continue unrolling for $d-1$-steps)}\\
    &= \lambda^{d-1} \max_{\v x \in \B^d} \u Q^*(\tau \vo{xor_\bot}_{<d}, \v x_d) \numberthis
    \eqan
    where $(a)$ follows from the definition of $\u P$ and the fact that $r' = r_\bot = 0$ when $\u P \neq 0$.
\end{proof}

Now, using \Cref{prep:expandsion} we can prove a relationship between the action-value functions of the actual environment and the sequentialized environment.

\begin{lemma}[$\u Q^*$ $\v x$-relationship]\label{lem:q-relation}
    For any history $h$ with the corresponding sequentialized history $\tau = g(h)$ and $\B$-ary decision vector $\v x \in \B^d$, the following holds for any $i \leq d$.
    \beqn
    \u Q^*(\tau\vo{xor_\bot}_{<i}, \v x_i) = \g^{\frac{d-i}{d}} Q^*(h,\Pi^*(h,\v{x}_{\leq i}))
    \eeqn
\end{lemma}
\begin{proof}
    Before we prove the general result, we show that the result holds for $i = d$, i.e.\ the sequentialized problem has same optimal action-values at the ``real'' decision steps. Note that $\Pi^*(h,\v x_{\leq d}) = D(\v x)$. Let $\v x := C(a)$ and $\tau := g(h)$. Using the fact that $r_\bot = constant = 0$, we get
    \bqan
    f_{r_\bot}(h,a)
    &:= \u Q^*(\tau \vo{xor_\bot}_{< d}, \v x_d) \\
    &\overset{(a)}{=} \sum_{o'r'} \u P(o'r'|\tau \vo{xor_\bot}_{< d}\v x_d) \left(r' + \lambda \max_{x'} \u Q^{*}(\tau \vo{xor_\bot}_{< d} \v x_d o'r', x')\right) \\
    &\overset{(b)}{=} \sum_{o'r'} P(o'r'|ha) \left(r' + \lambda \max_{x'} \u Q^{*}(\tau \vo{xor_\bot}_{< d} \v x_d o'r', x')\right) \\
    &\overset{(c)}{=} \sum_{o'r'} P(o'r'|ha) \left(r'  + \lambda^d \max_{\v x' \in \B^d} \u Q^{*}(\tau \vo{xor_\bot}_{< d} \v x_d o'r' \vo{xor_\bot}'_{<d}, \v x'_d)\right) \\
    &\overset{(d)}{=} \sum_{o'r'} P(o'r'|ha) \left(r' + \g \max_{a' \in \A} f_{r_\bot}(hao'r',a')\right)\numberthis\label{eq:obe-2}
    \eqan
    where $(a)$ is just \Cref{eq:ube} with the optimal policy, $(b)$ follows by \Cref{pro:uPtoP}, $(c)$ is given by \Cref{prep:expandsion}, $(d)$ is true by rearranging the argument, the definition of $f_{r_\bot}$ and by using the relation $\lambda^d = \g$. Note that \Cref{eq:obe-2} is the OBE of the original problem, see \Cref{eq:obe}. The solution of the OBE is unique \cite{Lattimore2014b}, hence $f_{r_\bot}$ is indeed $Q^*$.

    Having proved the claim for $i = d$, we show that it also holds for any $i < d$.
    \bqan
    \u Q^*(\tau\vo{xor_\bot}_{<i}, \v x_i)
    &\overset{(a)}{=} \sum_{o'r'} \u P(o'r'|\tau \vo{xor_\bot}_{< i}\v x_i) \left(r' + \lambda \max_{x_{i+1}} \u Q^{*}(\tau \vo{xor_\bot}_{< i} \v x_i o'r', x_{i+1})\right) \\
    &\overset{(b)}{=} \lambda \max_{x_{i+1}} \u Q^{*}(\tau \vo{xor_\bot}_{< i} \v x_i or_\bot, x_{i+1}) \\
    &\vdotswithin{=} \text{(continue unrolling for $d-i-1$-steps)}\\
    &\overset{}{=} \lambda^{d-i} \max_{x_{i+1}} \dots \max_{x_{d}} \u Q^{*}(\tau \vo{xor_\bot}_{< i} \v x_i or_\bot x_{i+1}or_\bot \dots x_{d-1}or_\bot, x_d) \\
    &\overset{(c)}{=} \lambda^{d-i} \max_{a \in \A(\v x_{\leq i})} Q^*(h, a) \numberthis
    \eqan
    where, again $(a)$ is \Cref{eq:ube} with the optimal policy, $(b)$ follows from the definition of $\u P$ and $r_\bot = 0$, and $(c)$ is true from the fact that the claim holds for $i=d$ and the maximum is over the restrictive set of actions.
\end{proof}

What we have proven so far is that the sequentialization scheme produces action-value functions which (at the ``partial'' histories) are rescaled versions of the original action-value function. They agree with the original $Q^*$ at the decision points (at the ``complete'' histories) where the sequentialized policy $\u \Pi$ completes an action code.

We also show that a similar relationship as proved in \Cref{lem:q-relation} holds for a fixed policy $\u\Pi$. However, we use a different proof method for the following lemma. Note that $\u \Pi$ induces a policy $\b \Pi$ on the original environment, which can trivially be expressed as follows:
\beq\label{eq:bpi}
\b \Pi(a|h) := \prod_{i=1}^d \u \Pi(\v x_i | \tau \vo{xor_\bot}_{<i}) =: \u \Pi(\v x | \tau)
\eeq
for any $a = D(\v x)$ and $\tau = g(h)$.

\begin{lemma}[$\u Q^{\u\Pi}$ $\v x$-relationship]\label{lem:fixed-policy}
    For any arbitrary policy $\u \Pi$ the following relationship is true:
    \beq
    \u Q^{\u\Pi}(\tau \vo{xor_\bot}_{<d}, \v x_d) = Q^{\b \Pi}(h, D(\v x))
    \eeq
    for any history $\tau = g(h)$ and $\v x \in \B^d$.
\end{lemma}
\begin{proof}
    Before we prove the main result of the lemma, we show that the following relationship holds for the value-functions of the sequentialized and the original environment:
    \beq\label{eq:v-v}
    V^{\u\Pi}(\tau) = \lambda^{d-1}V^{\b \Pi}(h)
    \eeq
    for any $\tau = g(h)$.
    We use a different argument than \Cref{lem:q-relation} to prove the above statement. Lets imagine the sequentialized environment is at the history $\tau = g(h)$. The agent starts to follow the policy $\u\Pi$. The following is the $(\text{expected-reward}, \text{discount-factor})$ sequence it generates from this history.
    \bqan
    &(0,\lambda^0),(0,\lambda^1),\dots,(0,\lambda^{d-2}),(\b r,\lambda^{d-1}),\\
    &(0,\lambda^d),(0,\lambda^{d+1}), \dots,(0,\lambda^{2d-2}),(\b r',\lambda^{2d-1}),\\ &(0,\lambda^{2d}), \dots
    \eqan
    where $\b r$ is the expected reward. The sum of the reward part of the above sequence returns $\u V^{\u\Pi}(\tau)$.
    Now, if we re-scale the discount part of the above sequence by $\lambda^{d-1}$ we get $V^{\b\Pi}(h)$ as the sum of the reward part.

    \bqan
    &(0,\lambda^{1-d}),(0,\lambda^{2-d}),\dots,(0,\lambda^{-1}),(\b r,\lambda^{0}),\\
    &(0,\lambda^1),(0,\lambda^{2}), \dots,(0,\lambda^{d-1}),(\b r',\lambda^{d}), \\
    &(0,\lambda^{d+1}), \dots
    \eqan
    which proves \Cref{eq:v-v} when $\lambda^d = \gamma$. Now, let $a := D(\v x)$.
    \bqan
    Q^{\u\Pi}(\tau\vo{xor_\bot}_{<d},\v x_d)
    &= \sum_{o'r'} \breve{P}(o'r'|\tau\vo{xor_\bot}_{<d}\v x_d) \left(r' + \lambda V^{\breve{\Pi}}(\tau\vo{xor_\bot}_{<d}\v x_do'r')\right) \\
    &\overset{(a)}{=} \sum_{o'r'} P(o'r'|ha) \left(r' + \lambda V^{\breve{\Pi}}(\tau\vo{xor_\bot}_{<d}\v x_do'r')\right) \\
    &\overset{\eqref{eq:v-v}}{=} \sum_{o'r'} P(o'r'|ha) \left(r' + \lambda^d V^{\b\Pi}(hao'r')\right) \\
    &= \sum_{o'r'} P(o'r'|ha) \left(r' + \g V^{\b\Pi}(hao'r')\right) = Q^{\b\Pi}(h,D(\v x))
    \eqan
    where $(a)$ is due to \Cref{pro:uPtoP}.
\end{proof}

The following theorem proves the usefulness of our sequentialization framework. We show that the optimal policy of the sequentialized environment is also optimal in the original environment when it is lifted back using the decoding function $D$.

\begin{theorem}[Sequentialization preserves $\eps$-optimality]\label{lem:uplift}
    Any $\lambda^{d-1}\eps$-optimal policy of the sequentialized environment is $\eps$-optimal in the original environment.
\end{theorem}
\begin{proof}
    Let $\u \Pi$ be an $\eps'$-optimal policy of the sequentialized environment, where $\eps' := \lambda^{d-1}\eps$. It implies the following:
    \beq
    \u V^*(\tau\vo{xor_\bot}_{<i}) - \u V^{\u \Pi}(\tau\vo{xor_\bot}_{<i}) \leq \eps'
    \eeq
    for any complete sequentialized history $\tau = g(h)$ and $\v x \in \B^{i-1}$ where $i \leq d$. Especially, we are interested in the case when $i=1$, i.e.\ values at the complete histories.
    \beq\label{eq:near-opt}
    \u V^*(\tau) - \u V^{\u \Pi}(\tau) \leq \eps'
    \eeq
    With simple algebra, we can show that the following relationship holds for the optimal policies of the sequentialized and original processes:
    \bqan
    \u V^*(\tau)
    &\overset{(a)}{=} \max_{x} \u Q^*(\tau, x) \\
    &\overset{(b)}{=} \lambda^{d-1} \max_{\v x \in \B^d} \u Q^*(\tau \vo{xor_\bot}_{<d}, \v x_d) \\
    &\overset{(c)}{=} \lambda^{d-1} \max_{\v x \in \B^d} Q^*(h, D(\v x)) = \lambda^{d-1} V^*(h) \numberthis\label{eq:opt-v-v}
    \eqan
    where $(a)$ is the definition of the value function, $(b)$ holds due to \Cref{prep:expandsion}, and $(c)$ is true by applying \Cref{lem:q-relation} for $i = d$.

    Now, by simply using \Cref{eq:v-v} and \Cref{eq:opt-v-v}, we can prove the claim.
    \beq
    V^*(h) - V^{\b \Pi}(h) \overset{(a)}{=} \lambda^{1-d}\left(\u V^*(\tau) - \u V^{\u \Pi}(\tau)\right) \overset{\eqref{eq:near-opt}}{\leq} \eps
    \eeq
    for any $\tau = g(h)$, where $(a)$ is due to \Cref{eq:v-v} and \Cref{eq:opt-v-v}.
\end{proof}

We are done formally defining the setup. In the next section we put everything together under the context of ESA to establish the validity of our sequentialization setup.

\section{Extreme State Aggregation}\label{sec:esa}

In the previous sections, we formalized the GRL problem with a sequentialized action-space. A GRL agent keeps the history of its interaction to decide the next action. The history grows with time but even worse is that, without more assumptions and/or abstractions, no history ever repeats \cite{Hutter2009}. This is a unique characteristic of the history-based setup which sets it apart from the standard RL \cite{Sutton2018}. It enables the GRL framework to cover from the extreme case of unique histories to the most restrictive scenarios of bandits. However, without abstractions, a GRL agent which assumes every history is unique is more of a theoretical artifact than a realizable algorithm. It is critical to note that our sequentialization scheme results in just like any other GRL agent. It also requires an abstraction map (or further structural assumptions) to provide an implementable algorithm. Usually, one starts by assuming some structure on the history set(s). After reviewing some, we will argue against all of them.

On one extreme we have unique histories and on the other end, typically, the environment distribution is assumed to be Markovian, i.e.\ for any $h$ and $a$, $P(o'r'|ha) \equiv P(o'r'|oa)$ for all $o'r'$ where $o$ is the most recent observation in $h$ \cite{Sutton2018}. This means that histories with the same most recent observation are members of the same class (or state).
This assumption provides a lot of structure on $\H$. The value functions become functions of the most recent observations. Note that \citet{Hutter2016} defines the Markovian assumption directly on histories, which is a bit weaker than what we have stated above.

Unfortunately, the Markovian assumption is too strong to be used in many real-world problems. We do not interact with the world based on just our recent observations. As general agents, we keep ``relevant'' historical events in memory to plan better in the future, sometimes optimally. Apart from some ``toy'' examples and (well-defined) games \cite{Mnih2015,Silver2016,Silver2018}, this assumption demands too much structure on the history space. So, what other assumption can we make? We can keep the Markovian structure but can weaken the assumption, significantly, by assuming that the agent is not able to observe the state directly. The agent may require a sufficiently long history of interaction to discern the \emph{hidden} Markovian state of the environment \cite{Kaelbling1998}. The class of problems this assumption models is known as partially observable Markov decision problems (POMDP). Almost all problems we care about can be modeled as POMDPs. However, POMDP solution methods are very demanding and the optimal behavior is not guaranteed to be learnable in general \cite{Pendrith1998}. We do not address this non-MDP class any further.

We focus on other important quantities in GRL formulation, e.g.\ $\Pi^*, V^*$, and $Q^*$, and make no direction assumption on $P$. This has been a subject of many works; \cite{Li2006,Abel2016} considered a unified abstraction framework by mapping states similar in value, \citet{Hutter2016} subsumed the previous work by considering the GRL setup, and \citet{Majeed2019} extended the work to state-action abstractions. The value functions provide natural criteria to group histories. The resultant structure can be non-Markovian. Such non-MDP abstractions have many benefits over Markovian reductions. The resultant state-space ($\cong$ the set of groups of histories) can be significantly smaller with these abstractions than the Markovian counterparts. One advantage of using such abstractions, as compared to POMDPs, is the guarantee of the optimal behavior being a function of states, which helps the learning in many problems that were traditionally not considered learnable \cite{Majeed2018}.

However, the most remarkable aspect of such non-MDP abstractions is that there may exist an upper bound on the required number of states \emph{uniformly} for any problem, as it is the case in ESA \cite{Hutter2016}.
The idea is to group histories together which have \emph{similar} optimal action-values $Q^*$. Since $Q^*$ is a bounded real function (which is the case as $\R$ is bounded), we can potentially upper bound the required number of states by lumping together histories by discretization of the action-value function. In this work, we are primarily interested in non-MDP abstractions of the following type.

\begin{definition}[$\eps$-Q-uniform abstraction]
    An abstraction function $\phi: \H \to \S$ is an $\eps$-Q-uniform abstraction if for any $h, \d h \in \H$ and all $a \in \A$ we have
    \bqan
    &\left(\phi(h) = \phi(\d h)\right)\implies \abs{Q^*(h,a) - Q^*(\d h, a)} \leq \eps
    \eqan
    where $\S$ is the set of states\footnote{We consider a finite abstract set of states, but the underlying set of states ($\cong$ history-space) is (allowed to be) infinite. Note also that we consider \emph{approximate} Q-uniform aggregations which result into a finite abstract state.} of the abstraction.
\end{definition}

In ESA, the agent's policy is (constrained to be only) a function of the states. Although these states do not exhibit Markovian dynamics, the agent can ``pretend'' that the abstract process is Markovian. This structure provides a surrogate-MDP whose optimal policy is $\eps$-optimal in the original environment.

The $\eps$-Q-uniform, non-MDP abstractions lead to the following important result due to \citet{Hutter2016}. We only state the result without a proof for the closure of exposition, see \citet{Hutter2016} for more details about ESA and proofs.

In the following theorems we assume that the rewards are bounded in the unit interval, i.e.\ $\R \subseteq [0,1]$. This is done for brevity, and it is not a necessary condition. The rescaling of the rewards does not affect the decision-making process in (G)RL. In general, let the range of the rewards be $R := \max \R - \min \R$. Then, the scalars in the nominators of \Cref{thm:esa,thm:bin-esa} are replaced by $2R$ and $4R^2$ respectively.

\begin{theorem}[ESA {\cite[Theorem 11]{Hutter2016}}] \label{thm:esa}
    For every environment $P$ there exists a reduction $\phi$ and a surrogate-MDP whose optimal policy\footnote{See \citet{Hutter2016} of how to learn this policy, the surrogate-MDP, $Q^*$, and $\phi$.} is an $\eps$-optimal policy for the environment. The size of the surrogate-MDP is bounded (uniformly for any $P$) by\footnote{The 2 instead of a 3 in the original theorem is a trivial improvement by removing the grid point at 0 in the construction.}
    \beqn
    \abs{\S} \leq \left(\frac{2}{\eps (1-\g)^3}\right)^{\abs{\A}}
    \eeqn
\end{theorem}

This is a powerful result, but it suffers from the exponential dependence on the action-space size. We now put our action sequentialization framework to work and dramatically improve this dependency from exponential to only a logarithmic dependency in $\abs{\A}$.

So far, we have considered an arbitrary $\B$-ary decision set to sequentialize the action-space. However, in the following theorem we go to the extreme case of sequentializing the action-space to binary decisions ($\B = \SetB$) to squeeze out the maximum improvement possible through the framework.

\begin{theorem}[Binary ESA] \label{thm:bin-esa}
    For every environment there exists an abstraction and a corresponding surrogate-MDP for its binarized version ($\B = \SetB$) whose optimal policy is $\eps$-optimal for the true environment. The size of the surrogate-MDP is uniformly bounded for \emph{every} environment as
    \beqn
    \abs{\S} \leq \frac{4\ceil{1 - \g + \lb\abs{\A}}^6}{\g^2 \eps^2 (1-\g)^6}
    \eeqn
\end{theorem}
\begin{proof}
    Consider the agent that is interacting with the sequentialized/binarized environment $\u P$. By \Cref{lem:uplift}, we know that a near-optimal policy of this sequentialized environment is also near-optimal in the original environment. Now, if we use ESA on the binarized problem and get an $\eps'$-optimal policy through the surrogate-MDP by \Cref{thm:esa}, we are sured to be $\eps$-optimal in the original environment $P$ as explained above. Additionally, the size of the state-space is bounded as
    \beq\label{eq:bound}
    \abs{\S} \overset{\Cref{thm:esa}}{\leq} \left(\frac{2}{\eps' (1-\lambda)^3}\right)^2 = \frac{4}{{\eps'}^2 (1-\lambda)^6}
    \eeq
    where $\lambda$ is the discount factor of the sequentialized problem.
    Next, we upper bound \Cref{eq:bound} by using the fact that $\lambda^d = \g$. Let $\del := 1 - \g < 1$. So,
    \bqan
    1 - \lambda
    &= 1- (1-\del)^{1/d} = 1 - \e^{\frac{\ln (1-\del)}{d}}\\
    &\overset{(a)}{\geq} 1- \frac{1}{1-\ln (1-\del)/d}
    \overset{(b)}{\geq} 1 - \frac{1}{1+\del/d} \\
    &\overset{}{=} \frac{\del}{d + \del} = \frac{1-\g}{d + 1 - \g}
    \numberthis \label{eq:bound2}
    \eqan
    where $(a)$ holds due to $\frac{1}{\e^{-\alpha}} \leq \frac{1}{1-\alpha}$, $(b)$ is true by using the fact that $\del < 1$, hence $\ln(1-\del) \leq - \del$.
    Therefore, using \Cref{eq:bound}, \Cref{eq:bound2}, and $\eps' = \lambda^{d-1}\eps \geq \lambda^d \eps = \g \eps$
    we get,
    \beq
    \abs{\S} \leq \frac{4}{{\eps'}^2 (1-\lambda)^6} \leq  \frac{4 (1-\g+d)^6}{\g^2\eps^2 (1-\g)^6}
    \eeq
    which proves the claim.
\end{proof}

Superficially, it might seem that we have simply replaced the original discount factor with a larger value. But, it is not the case. If we simply scaled the discount factor (without sequentializing the actions) then the resulting bound would indeed deteriorate, see \Cref{thm:esa}, but on the contrary, with sequentialization/binarization and our analysis the bound (dramatically) improves.

Usually in RL the discount factor $\g$ is close to 1. In that case, the bound in \Cref{thm:bin-esa} can be tightened further as:
\beq
\abs{\S} \lesssim \frac{4\ceil{\lb\abs{\A}}^6}{\eps^2 (1-\g)^6}
\eeq
which agrees with the bound in \Cref{thm:esa} for the case when $\abs{\A} = 2$, i.e.\ when the original problem already has a binary action-space.

\section{Conclusion \& Outlook}\label{sec:conclusion}

This work contributes to the study of the GRL problem. We have provided a reduction to handle large state and action spaces by sequentializing the decision-making process.
This helped us improve the upper bound on the number of states in ESA from an exponential dependency in $\abs{\A}$ to logarithmic. The gain is \emph{double exponential} in terms of the action-space dependence at no other cost.

Our result carries a broader impact on the implementation of \emph{general} RL agents\footnote{The general (or strong) agents are designed to work with a wide range of environments \cite{Hutter2000}.}. The required storage for such agents, which have access to a non-MDP, approximate Q-uniform abstraction, can be reasonably bounded which only scales logarithmically in the size of the action-space.

We conclude the paper with some future research directions. This work analyses the case when the agent has a fixed aggregation map. \citet{Hutter2016} provides an outline for a learning algorithm to learn such abstractions which can be combined with our sequentialization framework.

Another direction, which we also did not touch in this work, is to explore the connection, if any, between the surrogate-MDPs of a map on the original environment, and its extension on the sequentialized problem. By lifting the small binary ESA map, say $\psi$, back to $\H$, one obtains a small map directly on $\H$, say $\phi$. While $\psi$ used sequentialization/binarization for the construction of $\phi$, the map $\phi$ can be used without further referencing to sequentialization. This suggests that a bound logarithmic in $\abs{\A}$ should be possible without a detour through the sequentialization. This deserves further investigation.

We sequentialize the action-space through an \emph{arbitrary} coding scheme $C$, so the main result does not depend on this choice. Sometimes, it is possible that the action-space may allow “natural” sequentialization, e.g.\ in a video game controller the “macro” action might be a binary vector where the first bit might represent the left/right direction, the second bit indicates up/down, and so on. The exact nature of these “binary decisions” depends on the domain which is reflected by the choice of encoding $C$. Sequentialization was our path to double-exponentially improve that bound. Whether there are more direct/natural aggregations with the same bound is an open problem. Moreover, if the agent is learning an abstraction through interaction, the choice of these functions may become critical.

This paper focused on rigorously formalizing and proving the main improvement result. One can also try to empirically show the effectiveness of our improved upper bound. To do this, we need a problem domain where ESA requires more states than the sequentialized/binarized version of it. But a point of caution is that the upper bound still scales badly in terms of $\g$ and $\eps$. Any reasonable value of these parameters would imply a huge upper bound.
Even with Markovian abstractions, a cubic dependency on the discount factor is the best achievable.
We considered a general underlying process and non-Markovian abstractions,
and dramatically improved the previously best bound $(1-\g)^{-3|\A|}$ to $(1-\g)^{-3\cdot 2}$.
Indeed it would be interesting to see whether this can be further improved to the optimal $(1-\g)^{-3}$ rate.

\paradot{Acknowledgements}
This work has been supported by Australian Research Council grant  DP150104590. Thanks to the anonymous reviewers for their feedback and to András György who pointed out that the sequentialized/binarized process in \Cref{fig:example} preserves the Markov property, which encouraged us to also consider the Markov case.

\begin{small}
    \bibliographystyle{\bibstylename}
    \bibliography{references}

\begin{thebibliography}{20}
\providecommand{\natexlab}[1]{#1}
\providecommand{\url}[1]{\texttt{#1}}
\expandafter\ifx\csname urlstyle\endcsname\relax
  \providecommand{\doi}[1]{doi: #1}\else
  \providecommand{\doi}{doi: \begingroup \urlstyle{rm}\Url}\fi

\bibitem[Sutton and Barto(2018)]{Sutton2018}
Richard~S. Sutton and Andrew~G. Barto.
\newblock \emph{{Reinforcement Learning: An Introduction}}.
\newblock MIT press Cambridge, 2nd edition, sep 2018.
\newblock ISBN 0262039249.

\bibitem[Powell(2011)]{Powell2011}
Warren~B. Powell.
\newblock \emph{{Approximate Dynamic Programming: Solving the Curses of
  Dimensionality: Second Edition}}, volume 136.
\newblock Wiley, 2011.
\newblock ISBN 9781118029176.
\newblock \doi{10.1002/9781118029176}.

\bibitem[Lattimore and Hutter(2014{\natexlab{a}})]{Lattimore2014}
Tor Lattimore and Marcus Hutter.
\newblock {Near-optimal PAC bounds for discounted MDPs}.
\newblock \emph{Theoretical Computer Science}, 558\penalty0 (C):\penalty0
  125--143, 2014{\natexlab{a}}.
\newblock ISSN 03043975.
\newblock \doi{10.1016/j.tcs.2014.09.029}.

\bibitem[Li et~al.(2006)Li, Walsh, and Littman]{Li2006}
Lihong Li, Thomas~J. Walsh, and Michael~L. Littman.
\newblock {Towards a unified theory of state abstraction for MDPs}.
\newblock \emph{9th International Symposium on Artificial Intelligence and
  Mathematics, ISAIM 2006}, pages 531--539, 2006.

\bibitem[Abel et~al.(2016)Abel, Hershkowitz, and Littman]{Abel2016}
David Abel, D.~Ellis Hershkowitz, and Michael~L. Littman.
\newblock {Near optimal behavior via approximate state abstraction}.
\newblock In \emph{33rd International Conference on Machine Learning, ICML
  2016}, volume~6, pages 4287--4295, jan 2016.
\newblock ISBN 9781510829008.

\bibitem[Hutter(2016)]{Hutter2016}
Marcus Hutter.
\newblock {Extreme state aggregation beyond Markov decision processes}.
\newblock \emph{Theoretical Computer Science}, 650:\penalty0 73--91, 2016.
\newblock ISSN 03043975.
\newblock \doi{10.1016/j.tcs.2016.07.032}.

\bibitem[Majeed and Hutter(2019)]{Majeed2019}
Sultan~Javed Majeed and Marcus Hutter.
\newblock {Performance Guarantees for Homomorphisms beyond Markov Decision
  Processes}.
\newblock \emph{Proceedings of the AAAI Conference on Artificial Intelligence},
  33:\penalty0 7659--7666, nov 2019.
\newblock ISSN 2159-5399.
\newblock \doi{10.1609/aaai.v33i01.33017659}.

\bibitem[Bertsekas and Tsitsiklis(1996)]{Bertsekas1996}
Dimitri Bertsekas and John Tsitsiklis.
\newblock \emph{{Dynamic programming: An overview}}, volume~35.
\newblock Belmont: Athena Scientific, 1996.
\newblock ISBN 0-7803-2685-7.
\newblock \doi{10.1109/cdc.1995.478953}.

\bibitem[Mnih et~al.(2015)Mnih, Kavukcuoglu, Silver, Rusu, Veness, Bellemare,
  Graves, Riedmiller, Fidjeland, Ostrovski, Petersen, Beattie, Sadik,
  Antonoglou, King, Kumaran, Wierstra, Legg, and Hassabis]{Mnih2015}
Volodymyr Mnih, Koray Kavukcuoglu, David Silver, Andrei~A. Rusu, Joel Veness,
  Marc~G. Bellemare, Alex Graves, Martin Riedmiller, Andreas~K. Fidjeland,
  Georg Ostrovski, Stig Petersen, Charles Beattie, Amir Sadik, Ioannis
  Antonoglou, Helen King, Dharshan Kumaran, Daan Wierstra, Shane Legg, and
  Demis Hassabis.
\newblock {Human-level control through deep reinforcement learning}.
\newblock \emph{Nature}, 518\penalty0 (7540):\penalty0 529--533, 2015.
\newblock ISSN 14764687.
\newblock \doi{10.1038/nature14236}.

\bibitem[Hutter(2009)]{Hutter2009}
Marcus Hutter.
\newblock {Feature Reinforcement Learning: Part I. Unstructured MDPs}.
\newblock \emph{Journal of Artificial General Intelligence}, 1\penalty0
  (1):\penalty0 3--24, jan 2009.
\newblock ISSN 1946-0163.
\newblock \doi{10.2478/v10229-011-0002-8}.

\bibitem[McCallum(1995)]{McCallum1995}
Andrew McCallum.
\newblock {Instance-Based Utile Distinctions for Reinforcement Learning with
  Hidden State}.
\newblock \emph{Machine Learning Proceedings 1995}, pages 387--395, 1995.
\newblock \doi{10.1016/b978-1-55860-377-6.50055-4}.

\bibitem[Strehl et~al.(2009)Strehl, Li, and Littman]{Strehl2009}
Alexander~L. Strehl, Hong Li, and Michael~L. Littman.
\newblock {Reinforcement learning in finite MDPs: PAC analysis}.
\newblock \emph{Journal of Machine Learning Research}, 10:\penalty0 2413--2444,
  2009.
\newblock ISSN 15324435.
\newblock \doi{10.1.1.153.9841}.

\bibitem[Watkins and Dayan(1992)]{Watkins1992}
Christopher J. C.~H. Watkins and Peter Dayan.
\newblock {Q-learning}.
\newblock \emph{Machine Learning}, 8\penalty0 (3-4):\penalty0 279--292, 1992.
\newblock ISSN 0885-6125.
\newblock \doi{10.1007/bf00992698}.

\bibitem[Lattimore and Hutter(2014{\natexlab{b}})]{Lattimore2014b}
Tor Lattimore and Marcus Hutter.
\newblock {General time consistent discounting}.
\newblock \emph{Theoretical Computer Science}, 519:\penalty0 140--154,
  2014{\natexlab{b}}.
\newblock ISSN 03043975.
\newblock \doi{10.1016/j.tcs.2013.09.022}.

\bibitem[Silver et~al.(2016)Silver, Huang, Maddison, Guez, Sifre, {Van Den
  Driessche}, Schrittwieser, Antonoglou, Panneershelvam, Lanctot, Dieleman,
  Grewe, Nham, Kalchbrenner, Sutskever, Lillicrap, Leach, Kavukcuoglu, Graepel,
  and Hassabis]{Silver2016}
David Silver, Aja Huang, Chris~J. Maddison, Arthur Guez, Laurent Sifre, George
  {Van Den Driessche}, Julian Schrittwieser, Ioannis Antonoglou, Veda
  Panneershelvam, Marc Lanctot, Sander Dieleman, Dominik Grewe, John Nham, Nal
  Kalchbrenner, Ilya Sutskever, Timothy Lillicrap, Madeleine Leach, Koray
  Kavukcuoglu, Thore Graepel, and Demis Hassabis.
\newblock {Mastering the game of Go with deep neural networks and tree search}.
\newblock \emph{Nature}, 529\penalty0 (7587):\penalty0 484--489, jan 2016.
\newblock ISSN 14764687.
\newblock \doi{10.1038/nature16961}.

\bibitem[Silver et~al.(2018)Silver, Hubert, Schrittwieser, Antonoglou, Lai,
  Guez, Lanctot, Sifre, Kumaran, Graepel, Lillicrap, Simonyan, and
  Hassabis]{Silver2018}
David Silver, Thomas Hubert, Julian Schrittwieser, Ioannis Antonoglou, Matthew
  Lai, Arthur Guez, Marc Lanctot, Laurent Sifre, Dharshan Kumaran, Thore
  Graepel, Timothy Lillicrap, Karen Simonyan, and Demis Hassabis.
\newblock {A general reinforcement learning algorithm that masters chess,
  shogi, and Go through self-play}.
\newblock \emph{Science}, 362\penalty0 (6419):\penalty0 1140--1144, 2018.
\newblock ISSN 10959203.
\newblock \doi{10.1126/science.aar6404}.

\bibitem[Kaelbling et~al.(1998)Kaelbling, Littman, and
  Cassandra]{Kaelbling1998}
Leslie~Pack Kaelbling, Michael~L. Littman, and Anthony~R. Cassandra.
\newblock {Planning and acting in partially observable stochastic domains}.
\newblock \emph{Artificial Intelligence}, 101\penalty0 (1-2):\penalty0 99--134,
  1998.
\newblock ISSN 00043702.
\newblock \doi{10.1016/S0004-3702(98)00023-X}.

\bibitem[Pendrith and McGarity(1998)]{Pendrith1998}
Mark~D Pendrith and Michael~J McGarity.
\newblock {An analysis of direct reinforcement learning in non-Markovian
  domains}.
\newblock \emph{Proceedings of the Fifteenth International Conference on
  Machine Learning}, pages 421--429, 1998.

\bibitem[Majeed and Hutter(2018)]{Majeed2018}
Sultan~Javed Majeed and Marcus Hutter.
\newblock {On Q-learning convergence for non-Markov decision processes}.
\newblock \emph{IJCAI International Joint Conference on Artificial
  Intelligence}, 2018-July:\penalty0 2546--2552, 2018.
\newblock ISSN 10450823.
\newblock \doi{10.24963/ijcai.2018/353}.

\bibitem[Hutter(2005)]{Hutter2000}
Marcus Hutter.
\newblock {Universal Artificial Intellegence}.
\newblock \emph{Ecml}, pages 226--238, apr 2005.
\newblock ISSN 0160-6972.
\newblock \doi{10.1007/b138233}.

\end{thebibliography}
\end{small}

\end{document}